\documentclass{article}
\usepackage{fullpage}
\pdfoutput=1
\usepackage[utf8]{inputenc} % allow utf-8 input
\usepackage[T1]{fontenc}    % use 8-bit T1 fonts
\usepackage{hyperref}       % hyperlinks
\usepackage{url}            % simple URL typesetting
\usepackage{booktabs}       % professional-quality tables
\usepackage{amsfonts}       % blackboard math symbols
\usepackage{nicefrac}       % compact symbols for 1/2, etc.
\usepackage{microtype}      % microtypography
\usepackage{ktmacros}

\newcommand{\vol}{\mathrm{vol}}
\newcommand{\eff}{\ensuremath{\mathcal{F}}}
\newcommand{\bdR}{B_d(0, R)}
\newcommand{\Dee}{\mathcal{D}}
\newcommand{\grad}{\nabla}
\newcommand{\slope}{s}
\newcommand{\bd}[2]{B_d(#1, #2)}

\newcommand{\round}{\texttt{round}}
\newcommand{\ud}{\mathrm{d}}

\newcommand{\deeflS}{\Dee_f^{\lambda, \domain}}
\newcommand{\deefl}{\Dee_f^{\lambda}}
\newcommand{\hypercube}[1]{\mathbb{H}_{#1}}
\newcommand{\Alg}{\mathcal{A}}
\newcommand{\domain}{\mathcal{X}}

\title{Computational Separations between Sampling and Optimization}
\author{%
  Kunal Talwar\\
  Google Brain\\
  Mountain View, CA\\
  \texttt{kunal@google.com} \\
}

\date{}
\begin{document}
\maketitle
\begin{abstract}
  Two commonly arising computational tasks in Bayesian learning are Optimization (Maximum A Posteriori estimation) and Sampling (from the posterior distribution). In the convex case these two problems are efficiently reducible to each other. Recent work~\citep{MaCJFJ18} shows that in the non-convex case, sampling can sometimes be provably faster. We present a simpler and stronger separation.
  We then compare sampling and optimization in more detail and show that they are provably incomparable: there are families of continuous functions for which optimization is easy but sampling is NP-hard, and vice versa.  Further, we show function families that exhibit a sharp phase transition in the computational complexity of sampling, as one varies the natural temperature parameter. Our results draw on a connection to analogous separations in the discrete setting which are well-studied.
\end{abstract}
\section{Introduction}
Given a a compact set $\domain \subseteq \Re^d$ and function $f : \domain \rightarrow \Re$, one can define two natural problems:
\begin{description}
  \item[Optimize($f, \domain, \eps$)]: Find $\vx \in \domain$ such that $f(\vx) \leq f(\vx') + \eps$ for all $\vx' \in \domain$.
  \item[Sample($f, \domain, \eta$)]: Sample from a distribution on $\domain$ that is $\eta$-close to  $\mu^\star(\vx) \propto \exp(-f(\vx))$.
\end{description}
These problems arise naturally in machine learning settings. When $f$ is the negative log likelihood function of a posterior distribution, the optimization problem corresponds to the Maximum A Posteriori (MAP) estimate, whereas the Sampling problem gives us a sample from the posterior. In this work we are interested in the computational complexities of these tasks for specific families of functions.% $f$.

When $f$ and $\domain$ are both convex, these two problems have a deep connection (see e.g. \cite{LovaszV06}) and are efficiently reducible to each other in a very general setting.  There has been considerable interest in both these problems in the non-convex setting. Given that in practice, we are often able to practically optimize certain non-convex loss functions, it would be appealing to extend this equivalence beyond the convex case. If sampling could be reduced to optimization for our function of interest (e.g. differentiable Lipschitz functions), that might allow us to design sampling algorithms for the function that are usually efficient in practice.
\citet{MaCJFJ18} recently showed that in the case when $f$ is not necessarily convex (and $\domain = \Re^d$), these problems are not equivalent. They exhibit a family of continuous, differentiable functions for which approximate sampling can be done efficiently, but where approximate optimization requires exponential time (in an oracle model {\em \`a la}~\citet{NemiroskiY}). In this work, we study the relationship of these two problems in more detail.

To aid the discussion, it will be convenient to consider a more general sampling problem where we want to sample with probability proportional to $\exp(-\lambda f(\vx))$ for a parameter $\lambda>0$. Such a scaling has no effect on the optimization problem, up to scaling of $\eps$. However changing $\lambda$ can signifcantly change the distribution for the sampling problem. In statistical physics literature, this parameter is the inverse temperature. For families $\eff$ that are invariant to multiplication by a positive scalar (such as the family of convex functions), this $\lambda$ parameter has no impact on the complexity of sampling from the family. We will however be looking at families of functions that are controlled in some way (e.g. bounded, Lipschitz, or Smooth) and do not enjoy such an invariance to scale.
E.g. in some Bayesian settings, each sample may give us a $1$-smooth negative log likelihood function, so we may want to consider the family $\mathcal{F}_{smooth}$ of $1$-smooth functions. Given $n$ i.i.d. samples, the posterior log likelihood would be $-n \overline{f}$, where $\overline{f}=\frac{1}{n} \sum_i f_i$ is in $\mathcal{F}_{smooth}$. The parameter $\lambda$ then corresponds naturally to the number of samples $n$.

This phenomenon of sampling being easier than optimization is primarily a ``high temperature'' or ``low signal'' phenomenon.  As $\lambda$ approaches infinity, the distribution $\exp(-\lambda f)$ approaches a point mass at the minimizer of $f$. This connection goes back to at least~\cite{KirkpatrickGV83}  and one can easily derive a quantitative finite-$\lambda$ version of this statement for many function families. \citet{MaCJFJ18} reconcile this with their separation by pointing out that their sampling algorithm becomes inefficient as $\lambda$ increases.

We first show a more elementary and stronger separation. We give a simple family of continuous Lipschitz functions which are efficiently samplable but hard even to approximately optimize. This improves on the separation in~\citet{MaCJFJ18} since our sampler is exact (modulo representation issues), and much simpler. The hardness of optimization here is in the oracle model, where the complexity is measured in terms of number of point evaluations of the function or its gradients.

While these oracle model separations rule out black-box optimization, they leave open the possibility of efficient algorithms that access the function in a different way. We next show that this hardness can be strengthened to an NP-hardness for an efficiently computable $f$.
This allows for the implementation of any oracle for $f$ or its derivatives. Thus assuming the Exponential Time Hypothesis~\citep{ImpagliazzoP01}, our result implies the oracle model lower bounds. Additionally, it rules out efficient non-blackbox algorithms that could examine the implementation of $f$ beyond point evaluations. We leave open the question of whether other oracle lower bounds \citep{NemiroskiY,Nesterov14,Bubeck15,Hazan16} in optimization can be strengthened to NP-hardness results.

We next look at the large $\lambda$ case. As discussed above, for large enough $\lambda$ sampling must be hard if optimization is. {\em Is hardness of optimization the only obstruction to efficient sampling?} We answer this question in the negative. We exhibit a family of functions for which optimization is easy, but where sampling is NP-hard for large enough $\lambda$. We draw on extensive work on the discrete analogs of these questions, where $f$ is simple (e.g. linear) but $\domain$ is a discrete set.

Our upper bound on optimization for this family can be strengthened to work under weaker models of access to the function, where we only have blackbox  access to the function. In other words, there are functions that can be optimized via gradient descent for which sampling is NP-hard. Conceptually, this separation is a result of the fact that finding one minimizer suffices for optimization whereas sampling intuitively requires finding {\em all} minima.

Both the separation result of~\cite{MaCJFJ18}, and our small-$\lambda$ result have the property that the sampling algorithm's complexity increases exponentially in $poly(\lambda)$. Thus as we increase $\lambda$, the problem gradually becomes harder. {\em Is there always a smooth transition in the complexity of sampling?}

Our final result gives a surprising negative answer. We exhibit a family of easy-to-optimize functions for which there is a sharp threshold: there is a $\lambda_c$ such that for $\lambda < \lambda_c$, sampling can be done efficiently, whereas for  $\lambda > \lambda_c$, sampling becomes NP-hard.
In the process, this demonstrates that for some families of functions, efficient sampling algorithms can be very structure-dependent, and do not fall into the usual Langevin-dynamics, or rejection-sampling categories.

Our results show that once we go beyond the convex setting, the problems of sampling and optimization exhibit a rich set of computational behaviors.
The connection to the discrete case should help further understand the complexities of these problems.

The rest of the paper is organized as follows.
We start with some preliminaries in Section~\ref{sec:prelims}.
We give a simple separation between optimization and sampling in Section~\ref{sec:simple_sep} and derive a computational version of this separation.
Section~\ref{sec:discrete_to_cont} relates the discrete sampling/optimization problems on the hypercube to their continuous counterparts, and uses this connection to derive NP-hardness results for sampling for function families where optimization is easy.
In Section~\ref{sec:sharp_threshold}, we prove the sharp threshold for $\lambda$.
We describe additional related work in Section~\ref{sec:related}.
Some missing proofs and strengthenings of our results are deferred to the Appendix.

\section{Preliminaries}
\label{sec:prelims}

We consider real-valued functions $f : \Re^d \rightarrow  \Re$. We will be restricting ourselves to functions that are continuous and bounded. We say a function $f$ is $L$-Lipschitz if $f(\vx) - f(\vx') \leq L \cdot \|\vx - \vx'\|$ for all $\vx, \vx' \in \Re^d$. In this work, $\|\cdot\|$ will denote the Euclidean norm.

We will look at specific families of functions which have compact representations, and ask questions about efficiency of optimization and sampling.  We will think of $d$ as a parameter, and look at function families such that at any function in the family can be computed in $poly(d)$ time and space.

We will look at constrained optimization in this work and our constraint set $\domain$ will be a Euclidean ball. Our hardness results however do not stem from the constraint set, and nearly all of our results can be extended easily to the unconstrained case.

Given  $\lambda >0$ and a function $f$, we let $\deeflS$ denote the distribution on $\domain$ with $\Pr[\vx] \propto \exp(-\lambda f(\vx))$. When $\domain$ is obvious from context, we will usually omit it and write $\deefl$. We will $Z_f^{\lambda, \domain}$ for the {\em partition function} $\int_{\domain} \exp(-\lambda f(\vx)) \ud \vx$.

We will also look at real-valued functions $h : \hypercube{d} \rightarrow \Re$, where $\hypercube{d} = \{-1, 1\}^d$ is the $d$-dimensional hypercube. We will often think of a $\vy \in \hypercube{d}$ as being contained in $\Re^d$.
Analagous to the Euclidean space case, we define $\Dee_{h}^{\lambda,\hypercube{d}}$ as the distribution over the hypercube with $\Pr[\vy] \propto \exp(-\lambda h(\vy))$, and define $Z_{h}^{\lambda,\hypercube{d}}$ to be  $\sum_{\vy \in \hypercube{d}} \exp(-\lambda h(\vy))$.

We say that an algorithm $\eta$-samples from $\deefl$ if it samples from a distribution that is $\eta$-close to $\deefl$ in statistical distance. We will also use the {\em Wasserstein} distance between distributions on $\Re^d$, defined as $\mathcal{W}(P, Q) \stackrel{eq}= \inf_{\pi} \Ex_{(\vx, \vx') \sim \pi}[\|\vx - \vx'\|_2]$, where the $\inf$ is taken over all couplings $\pi$ between $P$ and $Q$.
We remark that our results are not sensitive to the choice of distance between distributions and extend in a straightforward way to other distances. As is common in literature on sampling from continuous distributions, we will for the most part assume that we can sample from a real-valued distribution such as a Gaussian and ignore bit-precision issues. Statements such as our NP-hardness results usually require finite precision arithmetic. This issue is discussed at length by~\cite{ToshD19} and following them, we will discuss using Wasserstein distance in those settings.
An $\eps$-optimizer of $f$ is a point $\vx \in \domain$ such that $f(\vx) \leq f(\vx') + \eps$ for any $\vx' \in \domain$.

In Appendix~\ref{app:folklore}, we quantify the folklore results showing that sampling for high $\lambda$ implies approximate optimization. Quantitatively, they say that for $L$-Lipschitz functions over a ball of radius $R$, sampling implies $\eps$-approximate optimization if $\lambda \geq \Omega(d \ln \frac{LR}{\eps} / \eps)$.
Similarly, for $\beta$-smooth functions over a ball of radius $R$, $\lambda \geq \Omega(d \ln \frac{\beta R}{\eps} / \eps)$ suffices to get $\eps$-close to an optimum.

\section{A Simple separation}
\label{sec:simple_sep}
We consider the case when $\domain = B_d(\mathbf{0}, 1)$ is the unit norm Euclidean ball in $d$ dimensions. We let $\mathcal{F}_{Lip}$ be the family of all $1$-lipschitz functions from $\domain$ to $\Re$.
We show that for any $f \in \mathcal{F}_{Lip}$, exact sampling can be done in time $\exp(O(\lambda))$. On the other hand, for any algorithm, there is an $f \in\mathcal{F}_{Lip}$ that forces the algorithm to use $\Omega(\lambda/\eps)^d$ queries to $f$ to get an $\eps$-approximate optimal solution.
Thus, e.g., for constant $\lambda$, sampling can be done in $poly(d)$ time, whereas optimization requires time exponential in the dimension. Moreover, for any $\lambda \leq d$, there is an exponential gap between the complexities of these two problems.
Our lower bound proof is similar to the analagous claim in~\cite{MaCJFJ18}, but has better parameters due to the simpler setting. Our upper bound proof is signifantly simpler and gives an exact sampler.

\begin{theorem}[Fast Sampling]
  There is an algorithm that for any $f \in \mathcal{F}_{Lip}$, outputs a sample from $\deefl$ and makes an expected $O(\exp(2\lambda))$ oracle calls to computing $f$.
\end{theorem}
\begin{proof}
  The algorithm is based on rejection sampling. We first compute $f(\mathbf{0})$ and let $M = f(\mathbf{0}) - 1$. By the Lipschitzness assumption, $f(\vx) \in  [M, M+2]$ for all $\vx$ in the unit ball. The algorithm now repeatedly samples a random point $\vx$ from the unit ball. With probability $\exp(\lambda(M-f(\vx)))$, this point is accepted and we output it. Otherwise we continue.

  Since $\exp(\lambda(M-f(\vx))) \in [\exp(-2\lambda), 1]$, this is a rejection sampling algorithm, and it outputs an exact sample from $\deefl$. Each step accepts with probability at least $\exp(-\lambda)$. Thus the algorithm terminates in an expected $O(\exp(2\lambda))$ many steps, each of which requires one evaluation of $f$.
  \end{proof}
  \begin{remark}
    The above algorithm assumes access to an oracle to sample from $B_d(\mathbf{0}, 1)$ to arbitrary precision. This can be replaced by sampling from a grid of finite precision points in the ball. This creates two sources of error. Firstly, the function is not constant in the grid cell. This error is easily bounded since $f$ is Lipschitz. Secondly, some grid cells may cross the boundary of $B_d(\mathbf{0}, 1)$. This is a probability $d 2^{-b}$ event when sampling a grid point with $b$ bits of precision. Taking these errors into account gives us a sample within Wasserstein distance at most $O((d+\lambda) 2^{-b})$.
  \end{remark}
  \begin{remark}
    The above is a Las Vegas algorithm. One can similarly derive a Monte Carlo algorithm by aborting and outputting a random $\vx$ after $\exp(2\lambda) \log \frac 1 \eta$ steps.
  \end{remark}
  \begin{remark}
    Under the assumptions in~\cite{MaCJFJ18} ($\beta$-smooth $f$, $\nabla(\mathbf{0})=\mathbf{0}$, $\kappa\beta$-strong convexity outside a ball of radius $R$), a direct reduction to our setting will be lossy and a rejection-sampling-based approach will not be efficient. The Langevin dynamics based sampler in that work is more efficient under their assumptions.
  \end{remark}

  \begin{theorem}[No Fast Optimization]
  \label{thm:no_fast_opt}
  For any algorithm $\Alg$ that queries $f$ or any of its derivatives at less than $(1/4\eps)^d$ points, there is an $f \in \mathcal{F}_{Lip}$ for which $\Alg$ fails to output an $\eps$-optimizer of $f$ except with negligible probability.
\end{theorem}
\begin{proof}
  Consider the function $f_{\vx}$ that is zero everywhere, except for a small ball of radius $2\eps$ around $\vx$, where it is $f(\vy) = \|\vy - \vx\| - 2\eps$. i.e. the function\footnote{As defined, this function is Lipschitz but not smooth. It can be easily modified to a $2$-Lipschitz, $2/\eps$-smooth function by replacing its linear behavior in the ball by an appropriately Huberized version.} is $f_{\vx}(\vy) = \min(0, \|\vy - \vx\| - 2\eps)$. This function has optimum $-2\eps$.  Let $g$ be the zero function.

  Let $\Alg$ be an algorithm (possibly randomized) that queries $f$ or its derivatives at $T \leq (1/4\eps)^d$ points. Consider running $\Alg$ on a  function $f$ chosen randomly as:
  \begin{align*}
    f &= \left\{ \begin{array}{ll}
    g & \mbox{ with probability $\frac 1 2$}\\
    f_{\vx} & \mbox{ for a $\vx$ chosen u.a.r. from $B(\mathbf{0}, 1)$ otherwise.}\\
  \end{array}\right.
    \end{align*}
  Until $\Alg$ has queried a point in $B(\vx, 2\eps)$, the behavior of the algorithm on $f_{\vx}$ and $g$ is identical, since the functions and all their derivatives agree outside this ball. Since an $\eps$-approximation must distinguish these two cases, for $\Alg$ to succeed, it must query this ball. The probability that $\Alg$ queries in this ball in any given step is at most $\frac{\vol(B(\vx, 2\eps))}{\vol(B(\mathbf{0}, 1))} = (2\eps)^d$.
  As $\Alg$ makes only $(1/4\eps)^d$ queries in total, the expected number of queries $\Alg$ makes to the ball $B(\vx, 2\eps)$ is at most $2^{-d}$. Thus with probability at least $1 - \frac{1}{2^d}$, the algorithm fails to distinguish $g$ from $f_\vx$, and hence cannot succeed.
\end{proof}

\subsection{Making the Separation Computational}
The oracle-hardness of Theorem~\ref{thm:no_fast_opt} stems from possibly ``hiding'' a minimizer $\vx$ of $f$. The computational version of this hardness result will instead possibly hide an appropriate encoding of a satisfying assignment to a 3SAT formula.
\begin{theorem}[No Fast Optimization: Computational]
  There is a constant $\eps > 0$ such that it is NP-hard to $\eps$-optimize an efficiently computable Lipschitz function over the unit ball.
\end{theorem}
\begin{proof}
  Let $\phi : \{0, 1\}^d \rightarrow \bd{\mathbf{0}}{1}$ be a map such that $\phi$ is efficiently computable, $\|\phi(\vy) - \phi(\vy')\| \geq 4\eps$ and such that given $\vx \in \bd{\phi(\vy)}{2\eps}$, we can efficiently recover $\vy$. For a small enough absolute constant $\eps$, such maps can be easily constructed using error correcting codes and we defer details to Appendix~\ref{app:eccs}.

    We start with an an instance $I$ of 3SAT on $d$ variables, and define $f$ as follows.
  Given a point $\vx$, we first find a $\vy \in \{0, 1\}^d$, if any, such that $\vx \in \bd{\phi(\vy)}{2\eps}$. If no such $\vy$ exists, $f(\vx)$ is set to $0$.
  If such a $\vy$ exists, we interpret it as an assignment to the variables in the 3SAT instance $I$ and set $f(\vx)$ to be $\min(0, \|\phi(\vy) - \vx\| - 2\eps)$ if $\vy$ is a satisfying assignment to instance $I$, and to $0$ otherwise.

  It is clear from definition that $f$ as defined is efficiently computable. Moreover, the minimum attained value for $f$ is $-2\eps$ if $I$ is satsifiable, and $0$ otherwise. Since distinguishing between these two cases is NP-hard, so is $\eps$-optimization of $f$.
\end{proof}

We note that assuming the exponential time hypothesis, this implies the $\exp(\Omega(d))$ oracle complexity lower bound of Theorem~\ref{thm:no_fast_opt}.

\section{Relating Discrete and Continuous Settings}
\label{sec:discrete_to_cont}
For any function $h$ on the hypercube, we can construct a function on $f$ on $\Re^d$ such that optimization of $f$ and $h$ are reducible to each other, and similarly sampling from $f$ and $h$ are reducible to each other. This would allow us to use separation results for the hypercube case to establish analagous separation results for the continuous case.

\begin{figure*}
\begin{center}
\includegraphics[width=.48\textwidth]{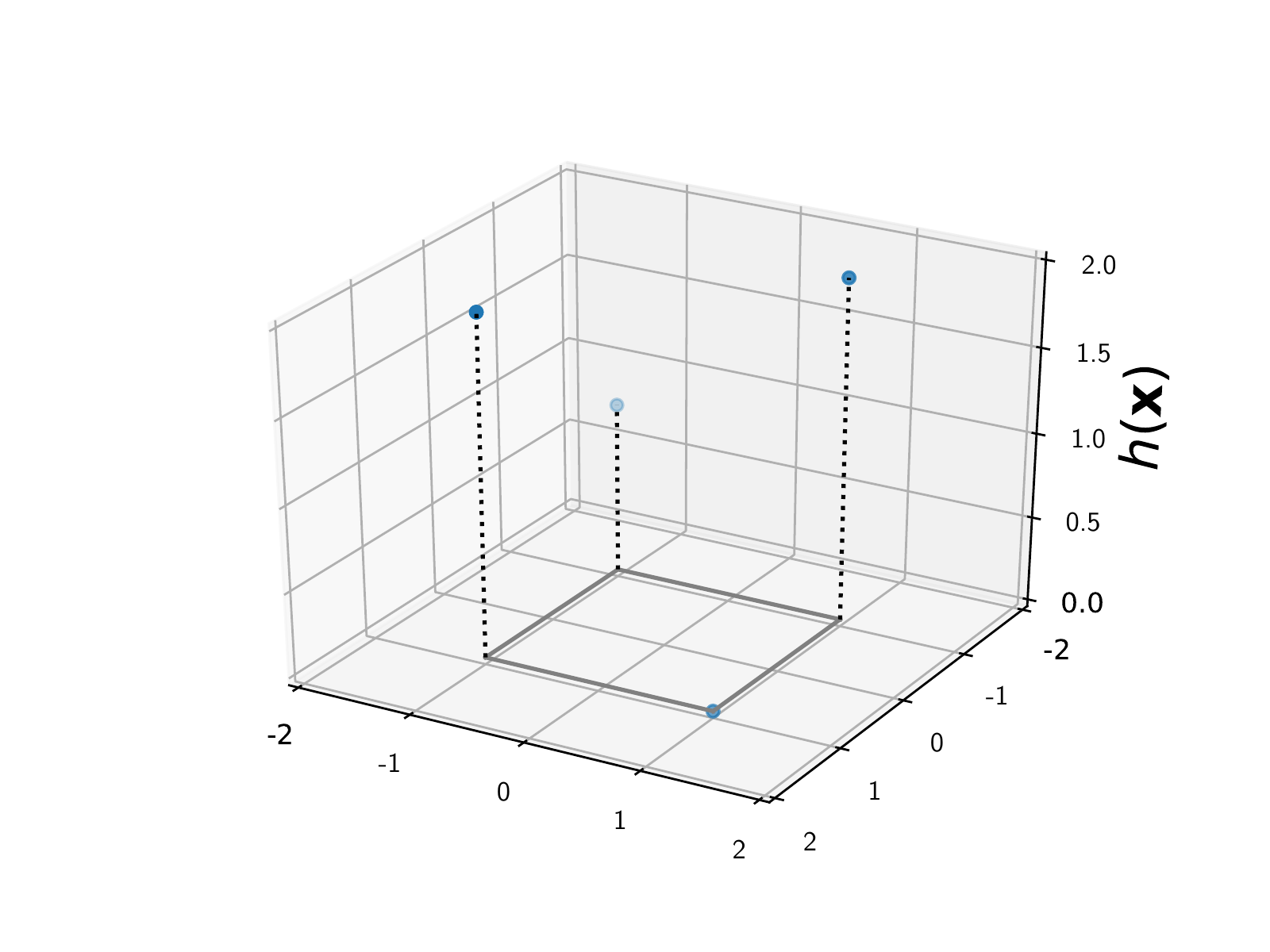}
\includegraphics[width=.48\textwidth]{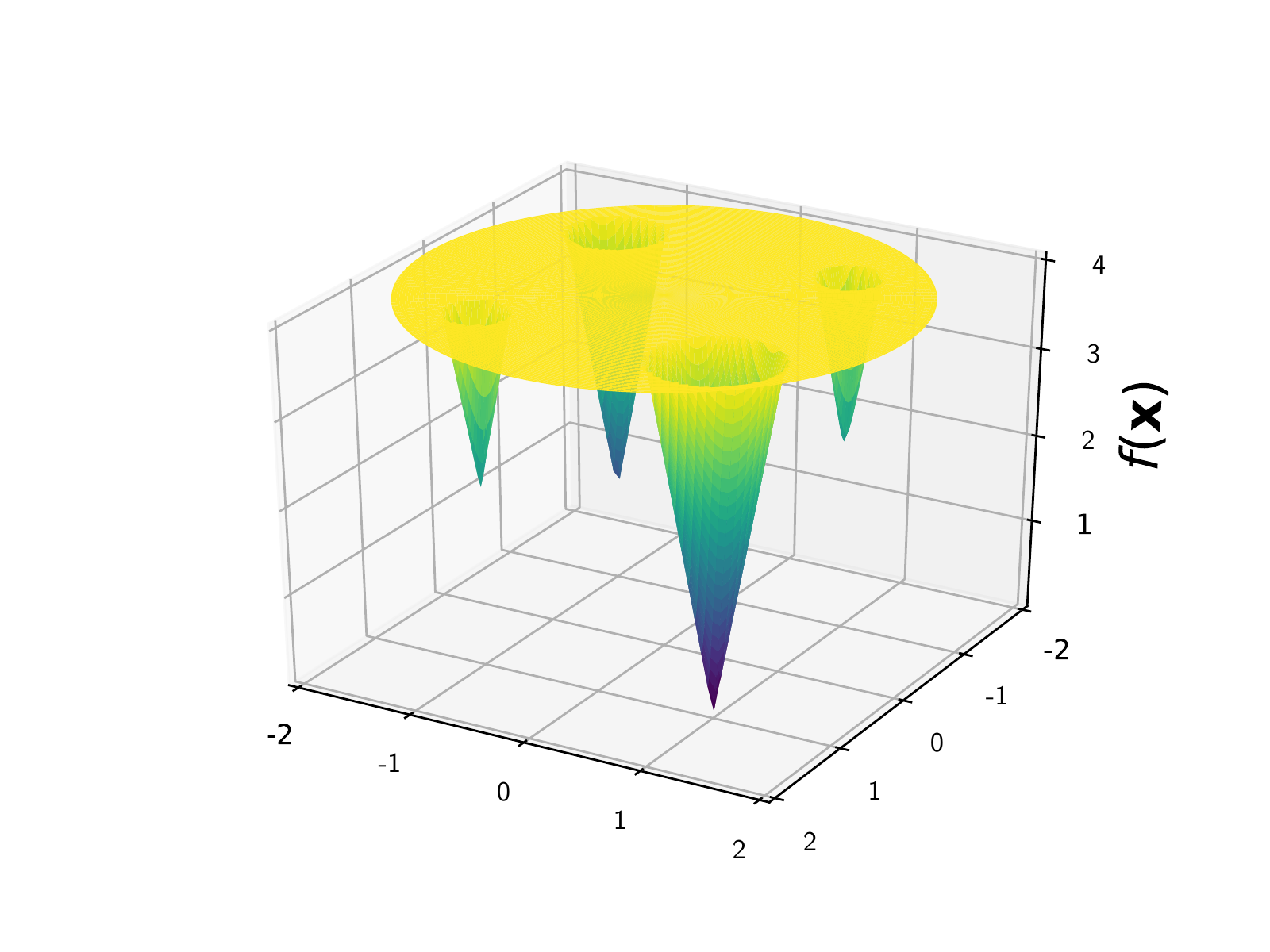}
\end{center}
\caption{(Left) An example of a function $h$ for $d=2$, with the $2$-d hypercube shown in gray and the values of $h$ denoted by blue points. (Right) The corresponding function $f$ that results from the transformation, for $M=4, R=2$.}
\label{fig:functions_h_f_new}
\end{figure*}
\begin{theorem}
  \label{thm:discrete_to_cont_v1}
  Let $h : \hypercube{d} \rightarrow \Re$ have range $[0, d]$. Fix $M \geq 2d, R \geq 2\sqrt{d}$. Then there is a funtion $f : \bdR \rightarrow \Re$ satisfying the following properties:
  \begin{description}
    \item[Efficiency:] Given $\vx \in \bdR$ and oracle access to $h$, $f$ can computed in polynomial time.
    \item[Lipschitzness:] $f$ is continuous and $L$-Lipschitz for $L= 2M$.
    \item[Sampler Equivalence:] Fix $\lambda \geq \frac{4d \ln 24R}{M}$. Given access to an $\eta$-sampler for $\Dee_h^{\lambda, \hypercube{d}}$, there is an efficient $\eta'$-sampler for $\deefl$, for $\eta' = \eta + \exp(-\Omega(d))$.
    Conversely, given access to an $\eta$-sampler for $\deefl$, there is an efficient $\eta'$-sampler for $\Dee_h^{\lambda, \hypercube{d}}$ for $\eta' = \eta + \exp(-\Omega(d))$.
  \end{description}
\end{theorem}
\begin{proof}
  The function $f$ is fairly natural: it takes a large value $M \geq 2d$ at most points, except in small balls around the hypercube vertices. At each hypercube vertex, $f$ is equal to the $h$ value at the vertex, and we interpolate linearly in a small ball. See Figure~\ref{fig:functions_h_f} for an illustration.

Formally, let $\round : \Re \rightarrow \{-1, 1\}$ be the function that takes the value  $1$ for $x \geq 0$ and $-1$ otherwise, and let $\round : \Re^d \rightarrow \hypercube{d}$ be its natural vectorized form that applies the function coordinate-wise.
Let $g(\vx) = \|\vx - \round(\vx)\|$ denote the Euclidean distance from $\vx$ to $\round(\vx)$. Let $\slope = 2M$.
The function $f$ is defined as follows:
  \begin{align*}
    f(\vx) &= \left\{\begin{array}{ll}
                        %  h(\round(\vx)) & \mbox{ if } g(\vx) \leq \frac 1 4\\
                          h(\round(\vx)) + \slope\cdot g(\vx) & \mbox{ if }  g(\vx) \leq \frac{M - h(\round(\vx))}{\slope}\\
                          M & \mbox{ if } g(\vx) \geq \frac{M - h(\round(\vx))}{\slope}\\
                     \end{array}
              \right.
  \end{align*}
  It is easy to verify that $f$ is continuous. Moreover, since $M \geq 2d$, and $h$ has range $[0,d]$, the value $\frac{M - h(\round(\vx))}{\slope}$ is in the range $[\frac 1 4, \frac 1 2]$. It follows that $f$ takes the value $M$ outside balls of radius $\frac 1 2$ around the hypercube vertices, and is strictly smaller than $M$ in balls for radius $\frac 1 4$.

  Since $\round(\vx)$ is easy to compute, this implies that $f$ can be computed in polynomial time, using a single oracle call to $h$.
  Moreover it is immediate from the definition that the function $f$ has Lipschitz constant $s$.

  Note that $f(\vy) = h(\vy)$ for $\vy \in \hypercube{d}$ and $f(\vx) \geq h(\round(\vx))$ for all $\vx \in \bdR$. This implies that the minimum value of $f$ is the same as the minimum value of $h$, and indeed any ($\eps$-)minimizer $\vy$ of $h$ also ($\eps$-)minimizes $f$. Conversely, let $\vx$ be an $\eps$-minimizer of $f$. Since $h(\round(\vx)) \leq f(\vx)$, it follows that $\round(\vx)$ is an $\eps$-minimizer of $h$.

Finally we prove the equivalence of approximate sampling. Towards that goal, we define an intermediate distribution on $\bdR$. Let $\widehat{\deefl}$ be the distribution $\deefl$ conditioned on being in $\cup_{\vy \in \hypercube{d}} \bd{\vy}{\frac 1 4}$.

We first argue that $\eta$-samplability of $\widehat{\deefl}$ is equivalent to $\eta$-samplability of $\Dee_h^{\lambda, \hypercube{d}}$. Suppose that $X$ is a sample from $\widehat{\deefl}$. Then for any $\vy^\star \in \hypercube{d}$,
\begin{align*}
\Pr[X \in \bd{\vy^\star}{\frac 1 4}] &= \frac{\int_{\bd{\vy^\star}{\frac 1 4}} \exp(-\lambda f(\vx))\;\ud \vx}{ \sum_{\vy \in \hypercube{d}}\int_{\bd{\vy}{\frac 1 4}} \exp(-\lambda f(\vx))\;\ud \vx}\\
&= \frac{\exp(-\lambda h(\vy^\star)) \cdot \int_{\bd{\vy^\star}{\frac 1 4}} \exp(-\slope\lambda g(\vx)) \;\ud \vx}{\sum_{\vy \in \hypercube{d}} \exp(-\lambda h(\vy)) \cdot \int_{\bd{\vy}{\frac 1 4}} \exp(-\slope\lambda g(\vx)) \;\ud \vx}\\
&= \frac{\exp(-\lambda h(\vy^\star))}{\sum_{\vy \in \hypercube{d}}\exp(-\lambda h(\vy))}
\end{align*}
Thus $\round(X)$ is a sample from $\Dee_h^{\lambda, \hypercube{d}}$.
Conversely, the same calculation implies that given a sample $Y$ from $\Dee_h^{\lambda, \hypercube{d}}$, and a vector $Z \in \bd{\mathbf{0}}{\frac 1 4}$ sampled proportional to $\exp(-s\lambda \|\vz\|)$, $Y + Z$ is a sample from $\widehat{\deefl}$. Noting that $Z$ is a sample from a efficiently sample-able log-concave distribution completes the equivalence.

We next argue that $\deefl$ and $\widehat{\deefl}$ are $\exp(-\Omega(d))$-close as distributions. Since $\widehat{\deefl}$ is a conditioning of $\deefl$, this is equivalent to showing that nearly all of the mass of $\deefl$ lies in $\cup_{\vy \in \hypercube{d}} \bd{\vy}{\frac 1 4}$.
We write
\begin{align}
  Z_{\lambda}^f &= \int_{\bdR} \exp(-\lambda f(\vx)) \;\ud \vx\nonumber\\
  &\geq \sum_{\vy \in \hypercube{d}} \int_{\bd{\vy}{\frac 1 4}} \exp(-\lambda f(\vx)) \;\ud \vx \nonumber\\
  &= \sum_{\vy \in \hypercube{d}} \exp(-\lambda h(\vy)) \cdot \int_{\bd{\vy}{\frac 1 4}} \exp(-\slope\lambda g(\vx)) \;\ud \vx\nonumber\\
\end{align}
Let $\widehat{Z_{\lambda}^f}$ denote this last expression. We will argue that $Z_{\lambda}^f \leq (1+\exp(-\Omega(d)))\widehat{Z_{\lambda}^f}$. We write $Z_{\lambda}^f$ as
\begin{align*}
  \int_{\bdR} \exp(-\lambda f(\vx)) \;\ud \vx
  &\leq \sum_{\vy \in \hypercube{d}} \int_{\bd{\vy}{\frac 1 2}} \exp(-\lambda f(\vx)) \;\ud \vx + \int_{\bdR} \exp(-\lambda M) \;\ud\vx \\
    &\leq \underbrace{\sum_{\vy \in \hypercube{d}} \exp(-\lambda h(\vy))\cdot \int_{{\bd{\vy}{\frac 1 2}}} \exp(-\slope\lambda g(\vx)) \;\ud \vx}_{\text{(A)}} + \underbrace{\int_{\bdR} \exp(-\lambda M) \;\ud\vx}_{\text{(B)}} \\
\end{align*}
A simple calculation, formalized as Lemma~\ref{lem:incomplete_gamma} in Appendix~\ref{app:incomplete_gamma} shows that the integral $\int_{{\bd{\vy}{\frac 1 2}}} \exp(-\slope\lambda g(\vx))\;\ud\vx$ is within $(1+ 2\exp(-d))$ of $\int_{{\bd{\vy}{\frac 1 4}}} \exp(-\slope\lambda g(\vx))\;\ud\vx$ for $\slope\lambda > 16d$.
This implies that the term $(A)$ above is at most $(1+2\exp(-d))\widehat{Z_{\lambda}^f}$. To bound the second term $(B)$ above, we argue that a ball of radius $\frac 1 8$ around any single vertex $\vy$ of the hypercube contributes significantly more than than the term $(B)$. Indeed
\begin{align*}
  \int_{\bd{\vy}{\frac 1 8}} \exp(-\lambda f(\vx)) \;\ud \vx &\geq \int_{\bd{\vy}{\frac 1 8}} \exp(-\lambda (h(\vy) + \frac{\slope}{8})) \;\ud \vx\\
  &\geq \exp(-\lambda (d+\frac{M}{4})) \cdot (\frac 1 8)^d \int_{\bd{\mathbf{0}}{1}} \;\ud \vx
  &\geq \exp(-\lambda (\frac{3M}{4})) \cdot (\frac 1 8)^d \int_{\bd{\mathbf{0}}{1}} \;\ud \vx
\end{align*}
Whereas,
\begin{align*}
  \int_{\bdR} \exp(-\lambda M) \;\ud\vx &= \exp(-\lambda M) \cdot R^d \cdot \int_{\bd{\mathbf{0}}{1}} \;\ud \vx.
\end{align*}
Thus $(B)/\widehat{Z_{\lambda}^f}$ is at most $\exp(-\lambda M/4 + d \ln 8R)$.
 Under the assumptions on $\lambda$, it follows that (B) is at most $\exp(-\Omega(d))$ times  $\widehat{Z_\lambda^h}$. In other words, we have shown that $Z_{\lambda}^f$ is at most $(1+\exp(-\Omega(d))) \widehat{Z_{\lambda}^h}$.

\end{proof}
\begin{remark}
  \label{rem:sampling_wasserstein}
  The equivalence of sampling extends immediately to Wasserstein distance. Indeed given a sampler for $\Dee_h^{\lambda, \hypercube{d}}$, one gets a Wasserstein sampler for $\widehat{\deefl}$ by sampling from a simple isotropic log-concave distribution. A Wasserstein sampler for a ball suffices for this. Since $\mathcal{W}(P, Q)$ is bounded by the diameter times the statistical distance, this gives a $\eta + \exp(-\Omega(d))$ Wasserstein sampler for $\deefl$. Similarly, a $\eta$ Wasserstein sampler for $\deefl$ conditioned on the support of $\widehat{\deefl}$ immediately yields an $O(\eta)$-sampler for $\Dee_h^{\lambda, \hypercube{d}}$.
  Moreover, it is easy to check that this conditioning is on a constant probability event as long as $\eta < \frac 1 {16}$.
\end{remark}
\subsection{Optimization can be Easier than Sampling}
Given the reduction from the previous section, there are many options for a starting discrete problem to apply the reduction. We will start from one of the most celebrated NP-hard problems. The NP-hardness of Hamiltonian Cycle dates back  to~\citet{Karp72}.
\begin{theorem}[Hardness of {\sc HamCycle}]
  \label{thm:hamcycle_hardness}
  Given a constant-degree graph $G = (V, E)$, it is NP-hard to distinguish the following two cases.
  \begin{description}
    \item[{\sc Yes} Case:] $G$ has a Hamiltonian Cycle.
    \item[{\sc No} Case:] $G$ has no Hamiltonian Cycle.
  \end{description}
\end{theorem}

We can then amplify the gap between the number of long cycles in the two cases.
\begin{theorem}[{\sc \#Cycle} hardness]
  \label{thm:count_gap_amplification}
  Given a constant-degree graph $G = (V, E)$ and for $L= |V|/2$, it is NP-hard to distinguish the following two cases.
  \begin{description}
    \item[{\sc Yes} Case:] $G$ has at least $1 + 2^{L}$ simple cycles of length $L$.
    \item[{\sc No} Case:] $G$ has exactly one simple cycle $C^{(planted)}$ of length $L$, and no longer simple cycle.
  \end{description}
  Moreover, $C^{(planted)}$ can be efficiently found in polynomial time.
\end{theorem}
The proof of the above uses a simple extension of a relatively standard reduction (see e.g.~\citet{Vadhan02}) from Hamiltonian Cycle. Starting with an instance $G_1$ of Hamiltonian Cycle, we replace each edge by a two-connected path of length $t$, for some integer $t$. For $L=nt$, this gives us $2^{L}$ cycles of length $L$ for every Hamiltonian cycle in $G_1$. Moreover, any simple cycle of length $L$ must correspond to a Hamiltonian Cycle in $G_1$. We add to $G$ a simple cycle of length $L$ on a separate set of $L$ vertices. This ``planted'' cycle is easy to find, since it forms its own connected component of size $L$. A full proof is deferred to Appendix~\ref{app:count_gap_amplification}.

Armed with these, we form a function on the hypercube in $d = |E'|$ dimensions such that optimizing it is easy, but sampling is hard.
\begin{theorem}
  \label{thm:cycle_h}
  There exists a function $h : \hypercube{d} \rightarrow [0, d]$ satisfying the following properties.
  \begin{description}
    \item[Efficiency:] $h$ can be computed efficiently on $\hypercube{d}$.
    \item[Easy Optimization:] One can efficiently find a particular minimizer $\vy^{(planted)}$ of $h$ on $\hypercube{d}$.
    \item[Sampling is hard:] Let $\lambda \geq 2d$. It is NP-hard to distinguish the following two cases, for $L = \Omega(d)$:
    \begin{description}
      \item[{\sc Yes} Case:] $\Pr_{\vy \sim \Dee_h^{\lambda, \hypercube{d}}}[\vy = \vy^{(planted)}] \leq \frac 1 {2^L}$
      \item[{\sc No} Case:] $\Pr_{\vy \sim \Dee_h^{\lambda, \hypercube{d}}}[\vy = \vy^{(planted)}] \geq 1 - \frac 1 {2^L}$
    \end{description}
    In particular this implies that $1-\frac{1}{2^{L-2}}$-sampling from $\Dee_h^{\lambda, \hypercube{d}}$ is NP-hard.
  \end{description}
\end{theorem}
\begin{proof}
  Let $G = (V, E)$ a graph produced by Theorem~\ref{thm:count_gap_amplification} and let $d = |E|$.
  A vertex $\vy$ of the hypercube $\hypercube{d}$ is easily identified with a set $E_{\vy} \subset E$ consisting of the edges $\{e \in E : y_{e'} = 1\}$. The function $h_1(\vy)$ is equal to zero if $E_\vy$ does not define a simple cycle and is equal to the length of the cycle otherwise. To convert this into a minimization problem, we define $h(\vy) = d - h_1(\vy)$. It is immediate that a minimizer of $h$ corresponds to a longest simple cycle in $G$.

  Given a vertex $\vy$, testing whether $E_{\vy}$ is a simple cycle can done efficiently, and the length of the cycle is simply $|E_{\vy}|$. This implies that $h$ can be efficiently computed on $\hypercube{d}$. Further, since we can find the planted cycle in $G$ efficiently, we can efficiently construct a minimizer $\vy^{(planted)}$ of $h$.

  Suppose that $G$ has at least $(2^L+1)$ cycles of length $L$. In this case, the distribution $\Dee_h^{\lambda, \hypercube{d}}$ restricted to the minimizers is uniform, and thus the probability mass on a specific minimizer $\vy^{(planted)}$ is at most $\frac 1 {2^{L}+1}$. This also therefore upper bounds the probability mass on $\vy^{(planted)}$ in the %full distribution
  $\Dee_h^{\lambda, \hypercube{d}}$.

  On the other hand, suppose that planted cycle is the unique longest simple cycle in $G$. Thus the probability mass on $\vy^{(planted)}$ is at least $\exp(\lambda(d-L)) / Z_{h}^{\lambda, \hypercube{d}}$. Since every other cycle is of length at most $L-1$, and there are at most $2^d$ cycles, it follows that $\frac{Z_{h}^{\lambda, \hypercube{d}}}{\exp(\lambda(d-L))} \leq 1 + \frac{2^d}{\exp(\lambda)}$. For $\lambda \geq 2d$, this ratio is $(1+\exp(-d)) \leq (1 - \frac{1}{2^d})^{-1}$. The claim follows.
\end{proof}

We can now apply Theorem~\ref{thm:discrete_to_cont_v1} to derive the following result.
\begin{theorem}
  \label{thm:opt_easier}
  There is a family $\eff$ of functions $f : \bdR \rightarrow \Re$ such that the following hold.
  \begin{description}
  \item[Efficiency: ] Every $f \in \eff$ is computable in time $poly(d)$.
  \item[Easy Optimization: ] An optimizer of $f$ can be found in time $poly(d)$
  \item[Sampling is NP-hard: ] For $\lambda \geq 2d$ and $\eta < 1 - \exp(-\Omega(d))$, there is no efficient $\eta$-sampler for $\deefl$ unless $NP=RP$.
  \end{description}
\end{theorem}
\begin{remark}
  In the statement above, efficiently computable means that given a $t$-bit representation of $\vx$, one can compute $f(\vx)$ to $t$ bits of accuracy in time $poly(d, t)$.
\end{remark}
\begin{remark}
  The easy optimization result above can be considerably strengthened. We can ensure that $\mathbf{0}$ is the optimizer of $f$ and that except with negligible probability, gradient descent starting at a random point will converge to this minimizer. Further, one can ensure that all local minima are global and that $f$ is strict-saddle. Thus not only is the function easy to optimize given the representation, black box oracle access to $f$ and its gradients suffices to optimize $f$. We defer details to the Appendix~\ref{app:stronger_optimizability}.
\end{remark}
\begin{remark}
  The hardness of sampling holds also for Wasserstein distance $\frac 1 {16}$, given Remark~\ref{rem:sampling_wasserstein}.
\end{remark}

\section{A Sharp Threshold for $\lambda$}
\label{sec:sharp_threshold}
We start with the following threshold result for sampling from the Gibb's distribution on independent sets due to~\cite{Weitz06,SlyS12}.
\begin{theorem}
  For any $\Delta \geq 6$, there is a threshold $\lambda_c(\Delta) > 0$ such that the following are true.
  \begin{description}
    \item[FPRAS for small $\lambda$:] For any $\lambda < \lambda_c(\Delta)$, the problem of sampling independent sets with $\Pr[I] \propto \exp(\lambda|I|)$ on $\Delta$-regular graphs has a fully polynomial time approximation scheme.
    \item[NP-hard for large $\lambda$:] For any $\lambda >\lambda_c(\Delta)$, unless $NP=RP$, there is no fully polynomial time randomized approximation scheme for the problem of sampling independent sets with $\Pr[I] \propto \exp(\lambda|I|)$ on $\Delta$-regular graphs.
  \end{description}
\end{theorem}

\begin{theorem}
  There is a family $\eff$ of functions $f : \bdR \rightarrow \Re$ such that the following hold.
  \begin{description}
  \item[Efficiency: ] Every $f \in \eff$ is computable in time $poly(d)$.
  \item[Sampling has a threshold:] There is a constant $\lambda_c > 0$ such that for any $\frac 1 d < \lambda < \lambda_c$, there is a $poly(d/\eta)$-time $\eta$-sampler from from the distribution $\deefl$. On the other hand, for $\lambda > \lambda_c$, there is a constant $\eta' >0$ such that no polynomial time algorithm $\eta'$-samples from $\deefl$ unless $NP=RP$.
\end{description}
\end{theorem}
\begin{proof}
  For a graph $G = (V, E)$, the function $h$ on the hypercube $\hypercube{d}$ is defined in the natural way for $d=|V|$. We identify $\vy$ with $V_{\vy} = \{v_i : \vy_i = 1\}$, and set $h(\vy)=d - |V_\vy|$ if $V_{\vy}$ is an independent set in $G$ and to $d$ otherwise.
  We then apply Theorem~\ref{thm:discrete_to_cont_v1}, with $R=2\sqrt{d}$ and $M = 4d^2 \ln 24R$. For this value of $M$, the approximate equivalence between sampling from $\exp(\lambda |I|)$ and sampling from $\exp(-\lambda f)$ holds for $\lambda \geq \frac{1}{d}$.
  The claim follows.
\end{proof}

\section{Related Work}
\label{sec:related}
The problems of counting solutions and sampling solutions are intimately related, and well-studied in the discrete case. The class \#P was defined in~\citet{Valiant79}, where he showed that the problem of computing the permanenet of matrix was complete for this class. This class has been well-studied, and~\cite{Toda91} showed efficient exact algorithms for any \#P-complete problem would imply a collapse of the polynomial hierarchy. Many common problems in \#P however admit efficient approximation schemes, that for any $\eps>0$, allow for a randomized $(1+\eps)$-approximation in time polynomial in $n/\eps$. Such {\em Fully Polynomial Randomized Approximation Schemes} (FPRASes) are known for many problems in \#P, perhaps the most celebrated of them being that for the permananet of a non-negative matrix~\citep{JerrumSV04}.

These FPRASes are nearly always based on Markov Chain methods, and their Metropolis-Hastings~\citep{MetropolisRRTT53,Hastings70} variants. These techniques have been used both in the discrete case (e.g.~\citet{JerrumSV04}) and the continuous case (e.g.~\citet{LovaszV06}). The closely related technique of Langevin dynamics~\citep{RosskyDF78,RobertsS02,DurmusM17} and its Metropolis-adjusted variant are often faster in practice and have only recently been analyzed.

\bibliographystyle{plainnat}
\bibliography{refs}

\newpage
\appendix
\section{Efficient Net Construction}
\label{app:eccs}
We start with the existence of error correcting codes for worst case error (e.g.~\citet{Justesen72}).
\begin{theorem}
  There exists a universal constant $\alpha > 0$, and an efficiently computable mapping $\psi : \{0, 1\}^d \rightarrow \{0, 1\}^{4d}$ such that for any $\vy \neq \vy' \in \{0, 1\}^d$, the Hamming distance $d_H(\psi(\vy), \psi(\vy')) \geq 4\alpha d$.
  Further, there is an efficiently computable mapping $\texttt{Decode} : \{0, 1\}^{4d} \rightarrow \{0, 1\}^d$ such that for any $\vy \in \{0, 1\}^D$ and any  $\vw$ satisfying $d_H(\psi(\vy), \vw) \leq \alpha d$, $\texttt{Decode}(\vw)$ returns $\vy$.
  \label{thm:eccs}
\end{theorem}
Armed with this, we prove the desired net construction.
\begin{theorem}
  There is an absolute constant $\eps > 0$ and an efficiently computable map $\phi : \{0, 1\}^d \rightarrow \bd{\mathbf{0}}{1}$ such that (a) $\|\phi(\vy) - \phi(\vy')\| \geq 4\eps$ and (b) $\vx \in \bd{\phi(\vy)}{2\eps}$, we can efficiently recover $\vy$.
\end{theorem}
\begin{proof}
  The map $\phi(\vy)$ will be constructed by embedding $\psi(\vy)$ into $\bd{\mathbf{0}}{1}$.
  We first interpret $\psi(\vy)$ as a vector $\hat{\vy}$ in $\{0, 1, \ldots, 15\}^d$, with the $i$th co-ordinate of $\hat{\vy}$ being defined by the 4 bits $\psi(\vy)_{4i-3}, \psi(\vy)_{4i-2}, \psi(\vy)_{4i-1}, \psi(\vy)_{4i}$.
  The $i$th co-ordinate of $\phi(\vy)$ is set to $\hat{\vy}_i / 15\sqrt{d}$. It is easy to check that $\phi(\vy)$ as defined lies in the unit ball.
  Consider $\vy, \vy' \in \{0, 1\}^d$ such that $d_H(\psi(\vy), \psi(\vy')) \geq 4\alpha d$. This means that $\hat{\vy}$ and $\hat{\vy'}$ differ in at least $\alpha d$ locations.
  This in turn means that $\|\phi(\vy) - \phi(\vy')\| \geq \sqrt{\alpha}/15$.

  Let $\vx \in \bd{\phi(\vy)}{\eps}$. By Markov's inequality, $|\{i : |\phi(\vy)_i - \vx_i| \geq \frac{1}{30\sqrt{d}}\}|$ is at most $900\eps^2 d$. Thus if we greedily decode each co-ordinate of $\vx$ and expand that to a $4d$-bit vector, the decoding will be correct except in $3600\eps^2 d$ locations.
  Letting $\eps = \sqrt{\alpha} / 120$, and using properties of $\psi$, the claim follows.
\end{proof}

\section{Sampling to Optimization}
\label{app:folklore}

The following folklore theorem shows that sampling for high $\lambda$ implies approximate optimization.

\begin{lem}
  \label{thm:sampling_to_opt}
  [Sampling Implies Optimization: Generic]
  Let $f : \bdR \rightarrow \Re$ and suppose that for some $a \in \Re$, the level set $L_a = \{\vx \in \bdR : f(\vx) \geq a\}$ is measurable. Then for any $\eps>0$, and for $\lambda \geq \frac{\ln \frac 1 \delta + \ln \frac {\vol_d(\bdR)}{\vol_d(L_a)}}{\eps}$,
  \begin{align*}
    \Pr_{\vx \sim \Dee_{\lambda}^f}[f(\vx) \geq a + \eps] \leq \delta.
  \end{align*}
\end{lem}
\begin{proof}
  We write
  \begin{align*}
    \Pr_{\vx \sim \Dee_{\lambda}^f}[f(\vx) \geq a + \eps] &\leq \frac{\exp(-\lambda(a+\eps))\cdot \vol_d(\bdR)}{Z_f^{\lambda}}\\
    &\leq \frac{\exp(-\lambda(a+\eps))\cdot \vol_d(\bdR)}{\exp(\lambda a)\cdot \vol_d(L_a)}\\
    &= \frac{\exp(-\lambda\eps)}{\frac{\vol_d(L_a)}{\vol_d(\bdR)}}\\
    &\leq \delta.
  \end{align*}
\end{proof}

\begin{theorem}
  \label{thm:sampling_to_opt_lip}
  [Sampling Implies Optimization: Lipschitz $f$]
  Let $f : \bdR \rightarrow \Re$ be $L$-Lipschitz and let $f^\star = \min_{\vx \in \bdR} f(\vx)$. Then for any $\eps>0$ and for $\lambda \geq \frac{2d\ln \frac{4LR}{\eps} + 2\ln \frac 1 \delta}{\eps}$,
  \begin{align*}
    \Pr_{\vx \sim \Dee_)f^{\lambda}}[f(\vx) \geq f^\star + \eps] \leq \delta.
  \end{align*}
\end{theorem}
\begin{proof}
  Set $a = \eps / 2$, and note that by Lipschitz-ness of $f$, the level set $L_a$ contains a ball of radius $\eps/2L$ around the optimizer, so that $\frac{\vol_d(L_a)}{\vol_d(\bdR)}$ is at least $(\eps/4LR)^d$. Applying Theorem~\ref{thm:sampling_to_opt}, the claim follows.
\end{proof}
\begin{theorem}
  [Sampling Implies Optimization: Smooth $f$]
  Let $f : \bdR \rightarrow \Re$ be $\beta$-smooth and let $f^\star = \min_{\vx \in \bdR} f(\vx)$ be attained in the interior of $\bdR$. Then for any $\eps>0$ and for $\lambda \geq \frac{d\ln \frac {\beta R^2}{\eps} + 2\ln \frac 1 \delta}{\eps}$,
  \begin{align*}
    \Pr_{\vx \sim \Dee_)f^{\lambda}}[f(\vx) \geq f^\star + \eps] \leq \delta.
  \end{align*}
\end{theorem}
\begin{proof}
  Let the mimimum $f^\star$ be attained at $\vx^\star$; by assumption $\grad f$ at $\vx^\star$ is zero, and by $\beta$-smoothness, the level set $L_{\eps/2}$ contains a ball of radius $(2\eps/\beta)^\frac{1}{2}$ around $\vx^\star$. The claim follows.
\end{proof}

We note that $\delta = \frac 1 2$ suffices to ensure that efficient samplability implies efficient approximate optimization, since we can run the sampler multiple times.

\section{Deferred Proofs}

\subsection{Incomplete Gamma Functions}
\label{app:incomplete_gamma}
The following was used in the proof of Theorem~\ref{thm:discrete_to_cont_v1}.
\begin{lem}
  \label{lem:incomplete_gamma}
  Let $d$ be a positive integer and $\alpha \geq 16d$. Then
  \begin{align*}
    \int_{0}^{\frac 1 2} \exp(-\alpha r) r^{d-1}\;\ud r &\leq (1+2\exp(-d))\int_{0}^{\frac 1 4} \exp(-\alpha r) r^{d-1}\;\ud r
  \end{align*}
\end{lem}
\begin{proof}
 By definition of the Incomplete Gamma function, for any $a \geq 0$,
  \begin{align*}
    \int_{0}^{a} \exp(-\alpha r) r^{d-1}\;\ud r &=     \frac{1}{\alpha^d} \cdot \int_{0}^{a\alpha} \exp(-s) s^{d-1}\;\ud s\\
    &= \frac{\gamma(d, a\alpha)}{\alpha^d}.
  \end{align*}
  Translated to this language, we wish to bound
  \begin{align*}
    \frac{\gamma(d, \alpha/2)}{\gamma(d, \alpha/4)}.
  \end{align*}

We now bound
\begin{align*}
  1 - \frac{\gamma(d, \alpha/4)}{\gamma(d, \alpha/2)} &\leq
  1 - \frac{\gamma(d, \alpha/4)}{\Gamma(d)} \\
  &= \Pr_{X \sim \Gamma(d, 1)}[X \geq \alpha/4]\\
  &\leq \Pr_{X \sim \Gamma(d, 1)}[X \geq d + 3d]\\
  &\leq \exp(-d).
\end{align*}
Here in the first step, we have used the fact that the Incomplete Gamma function is the cdf of the corresponding Gamma distribution. The last step uses standard tail inequalities for sub-gamma distributions from~\citet[Section 2.4]{BoucheronLM13}.

We have thus shown that $\frac{\gamma(d, \alpha/4)}{\gamma(d, \alpha/2)} \geq 1 - \exp(-d)$. Thus $\gamma(d, \alpha/2) \leq (1-\exp(-d))^{-1} \gamma(d, \alpha/4)$. For $d \geq 1$, $(1-\exp(d))^{-1} \leq (1+2\exp(-d))$ and the claim follows.
\end{proof}

\section{Count Gap Amplification for Cycles}
\label{app:count_gap_amplification}
Theorem~\ref{thm:count_gap_amplification} follows from the NP-hardness of Hamiltonian cycle and the following reduction.
\begin{theorem}
  Given a graph $G = (V, E)$ and an integer $k$, there is a polynomial time algorithm that outputs a graph $G' = (V', E')$ such that
  \begin{description}
    \item[``Completeness'':] If $G$ has a Hamiltonian cycle, then $G'$ has at least $1+ 2^{tn}$ simple cycles of length $nt$.
    \item[``Soundness'':] If $G$ has no Hamiltonian cycle, then $G'$ has exactly one cycle of length $nt$, and no longer simple cycles.
    \item[Efficiency:] It is easy to find one cycle of length $nt$ in $G'$.
  \end{description}
\end{theorem}
\begin{proof}
The reduction replaces each edge of $G$ by a path of length $t$ with each edge on the path being duplicated. In addition, we add a new set of $tn$ vertices that form a cycle, to ensure that we always have once cycle of length $tn$. We give more details next.

Let $G=(V,E)$ and let $e=(u,v) \in E$. Our new vertex set $V' = V_1 \cup V_2$, where $V_1 = V \cup \{e_i : e \in E, i \in [t-1]\}$ and $V_2 = \{w_1,\ldots, w_{nt}\}$. The vertices in $V_2$ form a simple cycle, i.e.
$E_2 = \{(w_i, w_{i+1} : i \in [nt-1])\} \cup \{(w_{nt}, w_1)\}$. For every edge $e = (u, v)\in E$, $E'$ contains two copies of each edge set $E_1^e = \{(u, e_1), (e_t, v)\} \cup \{(e_i, e_{i+1}) : i \in [t-1]\}$.

It is easy to see that $G'$ always has one cycle of length $nt$ consisting of $E_2$ that can be found efficiently. Whenever $G$ has a Hamiltonian cycle, we can form a cycle of length $nt$, by following the paths corresponding to the edges used in the Hamilitonian cycle in $G$. At each step, we have a choice of two edges to choose from, since $E'$ has parallel edges. This gives us $2^{nt}$ such cycles, proving the ``completeness'' part of the theorem.

Finally note that for any simple cycle of length $nt$ on $V_1$, the projection of the cycle onto the vertices in $V$ is a simple cycle of length $n$, i.e. a Hamiltonian cycle.  This completes the proof of the ``soundness'' part of the theorem.
\end{proof}

\section{Stronger Optimizability}
\label{app:stronger_optimizability}
In this section, we show that the separation between optimization and sampling holds even for stronger notions of $f$ being optimizable.

\begin{figure*}
\begin{center}
\includegraphics[width=.48\textwidth]{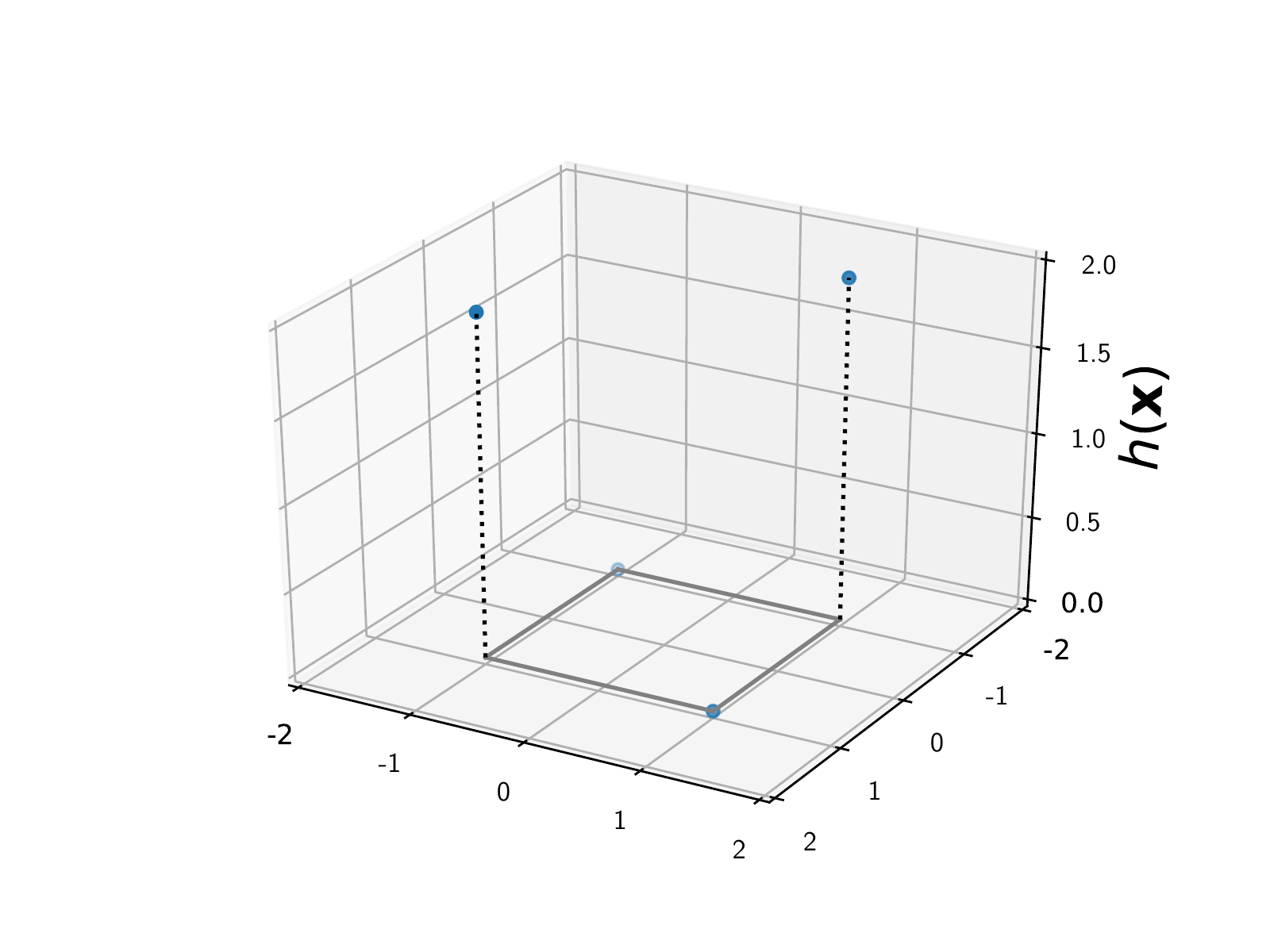}
\includegraphics[width=.48\textwidth]{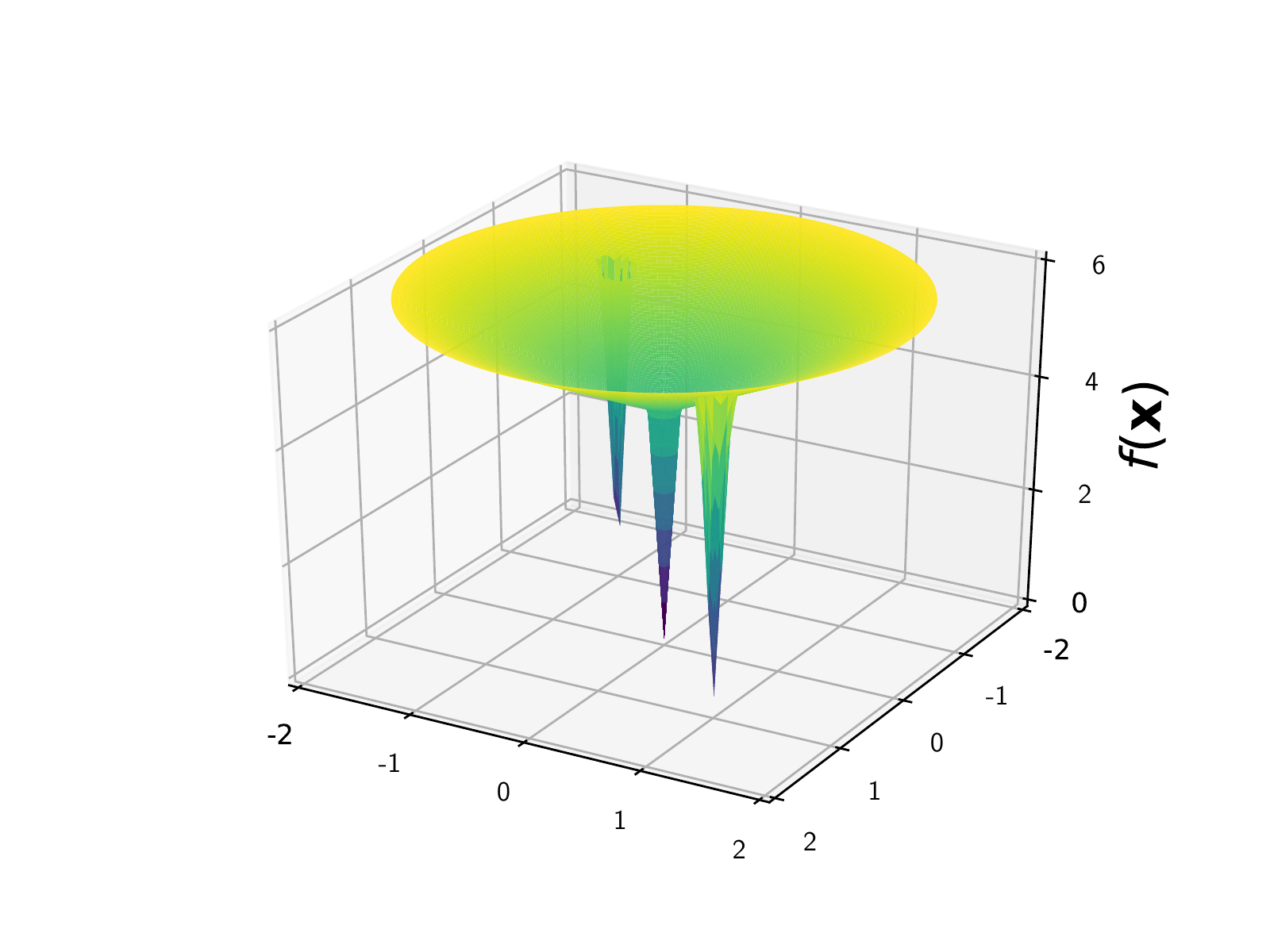}
\end{center}
\caption{(Left) An example of a function $h$ for $d=2$. (Right) The corresponding function $f$ that results from the transformation, for $M=4, R=2$. Note that we create a new minimizer at $\mathbf{0}$.}
\label{fig:functions_h_f}
\end{figure*}

\begin{theorem}
  There is a family $\eff$ of functions $f : \Re^d \rightarrow \Re$ such that the following hold.
  \begin{description}
    \item[Efficiency: ] Each $f \in \eff$ is computable in time $poly(d)$.
    \item[Easy Optimization: ] The zero vector $\mathbf{0}$ is a global optimizer of $f$. Further, $f$ satisfies strict saddle, and a randomly initialized gradient descent algorithm will converge to $\mathbf{0}$ with high probability.
    \item[Hard to Sample: ] For $\lambda \geq 2d$, there is no $1 - \exp(-\Omega(d))$-sampler for $\deefl$ unless $NP=RP$.
  \end{description}
  \label{thm:discrete_to_cont_v2}
\end{theorem}
\begin{proof}
  The proof is very similar to that of Theorem~\ref{thm:opt_easier}. We will highlight the relevant changes. First, we redefine $h$ slightly: $h$ now takes the value $0$ whenever $E_{\vy}$ defines a simple cycle of length $L$, and takes the value $d$ otherwise. Note that it is NP-hard to determines if $h^{-1}(0)$ has size $1$ or $2^L+1$.

  The construction from $h$ to $f$ is similar to that in Theorem~\ref{thm:discrete_to_cont_v1} with some variation. The function $f_1$ takes the value $M$, except in small balls around $h^{-1}(0)$ and a new minima at $0$.
  In addition, we add a linear term. Recalling the definition of $\round(\vx)$ and $g(\vx)$ from the proof of Theorem~\ref{thm:discrete_to_cont_v1}, we define $f(\vx)$ :
  \begin{align*}
    f(\vx) &= \left\{\begin{array}{ll}
                          \|\vx\| - \sqrt{d} + 16(M+\sqrt{d})\cdot g(\vx) & \mbox{ if } h(\round(\vx))=0 \mbox{ and } g(\vx) \leq  \frac 1 {16}\\
                          \|\vx\| + 16M\cdot \|\vx\| & \mbox{ if }  \|\vx\| \leq  \frac 1 {16}\\
                          \|\vx\| + M & \mbox{ otherwise }\\
                     \end{array}
              \right.
  \end{align*}
  It is easy to verify that $f$ satisfies the following properties:
  \begin{itemize}
    \item $f(\mathbf{0})$ = 0. For $\vy \in \hypercube{d}$, $h(\vy) =  0 \Leftrightarrow f(\vy) = 0$.
    \item $f$ is efficiently computable.
    \item $f$ is continuous and $O(M)$-Lipschitz.
    \item Outside of $\cup_{\vy : h(\vy = 0) \wedge \vy = \mathbf{0}} B_{d}(\vy, \frac 1 {16})$, $f$ is equal to $M+\|\vx\|$.
  \end{itemize}
It follows that the gradient at most points points towards the origin. Thus if the line joining the initial point of a gradient descent and $\mathbf{0}$ avoids hitting $\cup_{\vy : h(\vy = 0) } B_{d}(\vy, \frac 1 {16})$, gradient descent will converge to the origin. Since this happens with high probability, it follows that GD will succeed on $f$ w.h.p. While $f$ as defined is not twice differentiable, convolving $f$ with a small Gaussian gives us a function that is infinitely differentiable, and whose Lipschitz constant, and behavior with respect to gradient descent does not significantly change.

The hardness of sampling proof is essentially unchanged, and thus omitted.
\end{proof}

\section{Improving the Lipschitz constant}

The Lipschitz constant of the construction in Theorems~\ref{thm:discrete_to_cont_v1} and~\ref{thm:discrete_to_cont_v2} is $O(M)$, as the function value must change by $M$ in a ball of constant radius around $\round(\vx)$. We next argue how we can improve the Lipschitz constant to $O(M/\sqrt{d})$.

Towards that goal, we will use error correcting codes to embed the Hypercube in to a set of points that are at distance $\Omega(\sqrt{d})$ from each other. Recall that by Theeorem~\ref{thm:eccs} there are maps $\psi : \hypercube{d} \rightarrow \hypercube{4d}$ and $Decode :: \hypercube{4d} \rightarrow \hypercube{d}$ such that for a constant $\alpha > 0$,
(a) For all $\vy \neq \vy' \in \hypercube{d}$, the Hamming distance $|\psi(\vy) - \psi(\vy')| \geq 4\alpha d$, and
(b) For any $\vy \in \hypercube{d}$, and any $\vz \in \hypercube{4d}$ such that $|\vz - \psi(\vy)| \leq \alpha d$, we have $Decode(\vz) = \vy$.

Armed with this, we will now refine our function $\round$. For $\vx \in \Re^{4d}$, define
\begin{align*}
  \overline{\round}(\vx) = Decode(\round(\vx)).
\end{align*}
Thus $\overline{\round}(\vx) \in \hypercube{d}$. Moreover, there is a constant $\overline{\alpha}>0$ such that if $\vx \in B(\phi(\vy), \overline{\alpha}\sqrt{d})$, then $\overline{\round}(\vx) = \vy$.

Let $\overline{g}(\vx) = \|\vx - \overline{\round}(\vx)\| / (\overline{\alpha}\sqrt{d})$ denote the scaled Euclidean distance from $\vx$ to $\overline{\round}(\vx)$. Let $\slope = 2M$.
We now define the function $\overline{f}$ as follows:
  \begin{align*}
    \overline{f}(\vx) &= \left\{\begin{array}{ll}
                        %  h(\round(\vx)) & \mbox{ if } g(\vx) \leq \frac 1 4\\
                          h(\overline{\round}(\vx)) + \slope\cdot \overline{g}(\vx) & \mbox{ if }  \overline{g}(\vx) \leq \frac{M - h(\overline{\round}(\vx))}{\slope}\\
                          M & \mbox{ if } \overline{g}(\vx) \geq \frac{M - h(\overline{\round}(\vx))}{\slope}\\
                     \end{array}
              \right.
  \end{align*}
  Since $M \geq 2d$, and $h$ has range $[0,d]$, the value $\frac{M - h(\overline{\round}(\vx))}{\overline{\slope}}$ is in the range $[\frac 1 4, \frac 1 2]$. It follows that $f$ takes the value $M$ outside balls of radius $\overline{\alpha} \sqrt{d} / 2$ around the $\psi(\vy)$ vertices, and is strictly smaller than $M$ in balls for radius $\overline{\alpha}\sqrt{d} / 4$.
  Further, the Lipschitz constant of $\overline{f}$ is now $M/\sqrt{d}$. The rest of the proof is as before.

  In particular, this implies that assuming the Exponential Time Hypothesis~\citep{ImpagliazzoP01}, approximate sampling need $\exp(\Omega(d)$ time for continuous Lipschitz functions from $B_d(0, \sqrt{d})$ to $[0, O(d)]$, with Lipschitz constant $O(\sqrt{d})$, and $\lambda = O(\ln d)$.
  \end{document}